\documentclass[letterpaper, 10pt, conference]{ieeeconf}  

\usepackage{epsf}
\usepackage{epsfig}
\usepackage{times}
\usepackage{amssymb}
\usepackage{amsmath}
\usepackage{amsfonts}
\usepackage{latexsym}
\usepackage{url}
\usepackage{mathrsfs}
\usepackage{caption}
\usepackage{datetime}
\usepackage{bm}
\usepackage{algorithm,algorithmic}
\usepackage{subfigure}
\usepackage{framed} 
\usepackage[framed]{ntheorem}

\usepackage[noadjust]{cite}

\usepackage[T1]{fontenc}

\usepackage{amstext}  
\usepackage{array}      

\usepackage{mathabx}

\usepackage{color}

\newtheorem{definition}{Definition}

\newtheorem{theorem}{Theorem}
\newtheorem{lemma}{Lemma}
\newtheorem{assumption}{Assumption}
\newtheorem{remark}{Remark}
\usepackage{graphicx}
\usepackage{booktabs}

\IEEEoverridecommandlockouts                              

\begin{document}

\title{\LARGE \bf
Incentivized Exploration of Non-Stationary Stochastic Bandits
\thanks{S. Chakraborty \& L. Chen are from the computer science department at the University of Colorado Boulder, USA. The emails are \tt{ \{sourav.chakraborty, lijun.chen\}@colorado.edu}}

}

\author{Sourav Chakraborty and Lijun Chen} 

\maketitle
\thispagestyle{empty}

\begin{abstract}
We study incentivized exploration for the multi-armed bandit (MAB) problem with non-stationary reward distributions, where players receive compensation for exploring arms other than the greedy choice and may provide biased feedback on the reward. We consider two different non-stationary environments: abruptly-changing and continuously-changing, and propose respective incentivized exploration algorithms. We show that the proposed algorithms achieve sublinear regret and compensation over time, thus effectively incentivizing exploration despite the nonstationarity and the biased or drifted feedback.

\end{abstract}

\section{\textsc{Introduction}}

\label{sec:intro}

The multi-armed bandit (MAB) problem is one of basic models for sequential decision-making under uncertainty, with diverse applications in areas such as clinical trials \cite{gittins1, berry, william}, financial portfolio design \cite{brochu}, recommendation systems \cite{bouneffouf:hal-00753401, Li_2010}, search engine systems \cite{search-sys}, and cognitive radio networks \cite{cogradio}. In the traditional MAB model, a decision maker iteratively selects an arm (or action) to pull at each time step, receives a certain reward from the environment, and decides on the arm for the next iteration. In the so-called stochastic MAB model, each arm's reward distribution is unknown but remains fixed over time (hence, the `stationary' bandit setting).

The objective of the decision maker (or the MAB algorithm) is to minimize the expected regret over the entire time horizon, defined as the expectation of the difference between the total reward obtained by pulling the best arm and the total reward obtained by the algorithm. Minimizing regret is achieved by balancing exploitation, the use of acquired information, with exploration, acquiring new information. If the decision maker always pulls the arm believed to be the best (i.e., exploitation only), they may miss the opportunity to identify another arm with a potentially higher expected reward. On the other hand, if the decision maker excessively explores various arms (i.e., exploration only), they will fail to accumulate as much reward as possible.

In this setup, the decision maker (the principal) and the player (the agent) who pulls the arm are assumed to be the same entity striving to balance exploitation and exploration. However, this may not always be the case in the real world. Many scenarios exist where the principal and the agent are different entities with different interests. The agent may select the \textit{currently} best-performing arm in the face of uncertain reward (i.e., exploitation only), while the principal is interested in identifying the best-performing arm in the long run (thus, the need to balance exploration and exploitation). Consider, for instance, an e-commerce system like Amazon. Amazon (the principal) would like the customers (the agents) to buy and try different products (arms) to identify the revenue-maximizing product (i.e., the best-performing arm in the long run) for a particular search query. However, customers are influenced by the current ratings and reviews of the products and behave myopically, i.e., selecting the currently highest-rated product (exploitation only). Such exploitation-only behavior can lead to significantly degraded performance due to inadequate exploration, as demonstrated in previous studies \cite{bubeck2012regret, suttonbarto}. The misaligned interests between the principal and the agents need to be reconciled to balance exploration and exploitation optimally.

Incentivized exploration has been introduced to the MAB problem to reconcile different interests between the principal and the agents \cite{fraz, mansour, wang, immorlica2019bayesian}. The principal provides certain compensation to the agent to pull an arm other than the greedy choice currently having the best empirical reward, aiming to maximize the cumulative reward (or minimize the expected regret) while minimizing the total compensation to the agents. Early work on incentivized MAB models \cite{immorlica2019bayesian, wang, Hirnschall_Singla_Tschiatschek_Krause_2018, Han_2015, AAAI1816879} assumed that the agents provide unbiased feedback or reward, independent of compensation received. However, this assumption does not always hold in the real world, and experimental studies such as \cite{martensen, Ehsani2015EffectOQ} show that agents are inclined to give higher evaluation or reward with an incentive (such as a coupon, gift card, or discount in the case of Amazon). The compensation might even be the primary driver of customer satisfaction \cite{martensen, gwo}. This drift in reward feedback may negatively impact the exploration-exploitation tradeoff, as a suboptimal arm can be mistakenly identified as the optimal one because of the drifted rewards. Liu et al. \cite{liu20} investigated such an impact and showed that incentivized exploration based on their methods achieves optimal regret and compensation.

The authors in \cite{liu20} considered the \textit{stationary} bandit setting, i.e., the reward distribution of the arms does not change with time. In this paper, we consider the more challenging setting of \textit{nonstationary} bandits, corresponding to an evolving environment where the reward distribution changes over time. Consider again the Amazon example: In the stationary setting, a product, say, a snow boot (an arm), is assumed to have the same value to Amazon (the principal) in terms of sales throughout the year. However, a snow boot will be more valuable in winter than the summer. A specific product might gain sudden popularity due to celebrity endorsement or lose popularity because of a certain controversy. Such scenarios are common in the real world and require the consideration of nonstationary bandits. We aim to answer the following question: \textit{Can we achieve effective incentivized exploration, despite the non-stationarity of the reward and the drift in reward feedback?}

\noindent \\
\textbf{Contributions.} Specifically, we consider the incentivized exploration framework (\cite{liu20}, as illustrated in Fig. \ref{fig:ic}) where the agent receives from the principal a compensation that equals the difference in the estimated rewards between the principal’s recommended arm and the greedy choice, and provides biased feedback that is the sum of the true reward of an arm and a drift term that is a non-decreasing function of the compensation received for pulling the arm, but with changing reward distributions over time. We consider two non-stationary models and study the robustness of the proposed incentivized exploration algorithms in terms of regret and compensation. The first model assumes an abruptly changing environment where the rewards of each arm remain stationary until some breakpoint when they change abruptly. The second model considers a continuously changing environment where the rewards can vary continuously within a variation budget. We show that the regret and compensation bounds are sub-linear in time $T$, see Table \ref{summary-tab}, and thus the proposed algorithms effectively incentivize exploration in the non-stationary environment.

\begin{figure}
    \centering
    \includegraphics[width = 0.9\linewidth]{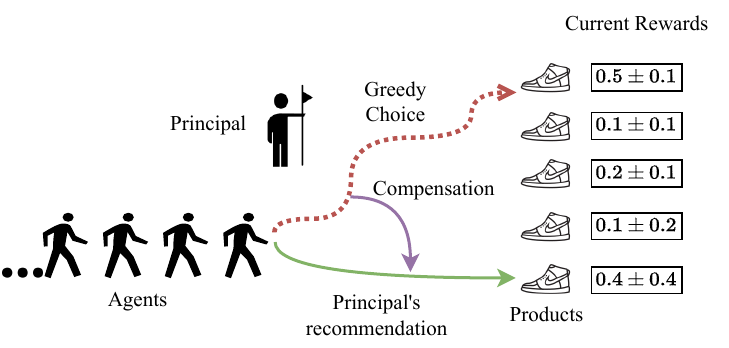}
    \caption{ Incentivized Exploration}
    \label{fig:ic}
\end{figure}

\begin{flushleft} 
\begin{table}
    \centering
    \caption{{\scriptsize Regret and compensation along with the corresponding theorems for the abruptly-changing(AC) and continuously-changing (CC) environments.}}
    \begin{tabular}{@{}lllll@{}}
    
    \toprule
    {\scriptsize Env.} & {\scriptsize Algo. Scheme} & {\scriptsize Theorem} & {\scriptsize Regret} & {\scriptsize Compensation}\\
    \midrule

    AC & Alg. \ref{inc-alg} + DUCB &  {\scriptsize \ref{thm-alg-ducb}, \ref{comp_alg1_ducb}} & {\footnotesize$\tilde{O}({T}^{1/2})$} & {\footnotesize$\tilde{O}({T}^{1/2})$}\\
    AC & Alg. \ref{inc-alg} + SWUCB &  {\scriptsize\ref{reg_alg1_swucb}, \ref{comp_alg1_swucb}} & {\footnotesize$\tilde{O}({T}^{1/2})$} & {\footnotesize$\tilde{O}(T^{1/4})$}\\
    CC & Alg. \ref{alg-cce} + UCB1 &  {\scriptsize \ref{reg_cce}, \ref{comp_cce}} & {\footnotesize$\tilde{O}(T^{2/3})$} & {\footnotesize$\tilde{O}(T^{2/3})$}\\
    CC & Alg. \ref{alg-cce} + $\epsilon$-Greedy &  {\scriptsize \ref{reg_cce}, \ref{comp_cce}} & {\footnotesize$\tilde{O}(T^{2/3})$} & {\footnotesize$\tilde{O}(T^{1/3})$}\\
    CC & Alg. \ref{alg-cce} + TS & {\scriptsize \ref{reg_cce}, \ref{comp_cce}} & {\footnotesize$\tilde{O}(T^{2/3})$} & {\footnotesize$\tilde{O}(T^{1/3})$}\\
    \bottomrule
    \end{tabular}
    \label{summary-tab}
\end{table}
\end{flushleft}

\noindent
\textbf{Related Work.} Early work on incentivized exploration and learning includes \cite{fraz, kremer, che} that introduced a Bayesian incentivized model with discounted regret and compensation, and \cite{mansour} that considered the non-discounted case and proposed an algorithm with $O(\sqrt{T})$ regret. In \cite{wang}, the authors analyzed the non-Bayesian and non-discount reward case and demonstrated $O(\log T)$ regret and compensation. In \cite{liu20}, the authors considered biased user feedback under the influence of incentives and showed that despite the reward drift, the proposed algorithms achieve $O(\log T)$ regret and compensation. 
Related work alos includes \cite{mansour2019bayesian, mansour, cohen, sellke2021price} on Bayesian Incentive Compatible (BIC) bandit exploration, where the principal wishes to persuade the agent to take action benefiting the principal, known as Bayesian Persuasion \cite{kamenica}. Additionally, see \cite{slivkins2021introduction} for a  review of the broad area of incentivized exploration. 

The paper is organized as follows: In Section \ref{prelim}, we introduce some preliminary concepts and results in stationary (Section \ref{mab}) and non-stationary MAB problems (Section \ref{nsmab}), as well as the incentivized exploration problem (Section \ref{inc-prereq}). In Section III, we present the proposed algorithms for incentivized exploration in abruptly-changing environments (Section \ref{ac-sol}) and continuously changing non-stationary environments (Section \ref{cc-sol}). In Section IV, we present  numerical experiment results for the proposed algorithms. The detailed theoretical analysis can be found in the Appendix (Section \ref{app}). 

\section{Preliminaries} \label{prelim}
\subsection{Standard stochastic (stationary) MAB Problem} \label{mab}
At each time step $t\in \{1, 2, \cdots, T\}$, a decision maker chooses to arm $a$ from a set of $K$ arms based on the sequence of past arm pulls and received rewards and obtains a reward $X_t(a)$. The rewards for each arm are modeled by a sequence of independent and identically distributed (i.i.d.) random variables from an unknown distribution. Without loss of generality, the reward of each arm is assumed to be in $[0,1]$. Denote by $\mu(a)$ the expectation of the reward of arm $a$, and $a^*$ the optimal arm with the highest expected reward $\mu^*$. The benchmark (or optimal) performance comprises pulling the optimal arm $a^*$ at every time step. The regret $R_T$ of an algorithm is defined as the difference between the benchmark performance and the total rewards collected by the algorithm:
\begin{align} \label{regret-def}
    R_T = \sum_{t=1}^T \left( \mu^* - X_t(a_t) \right),
\end{align} 
where $a_t$ is the arm the algorithm pulls at time $t$. A stochastic bandit algorithm's performance is typically evaluated by how the \textit{expected value} of $R_T$ scales with the time horizon $T$, and the goal is to design algorithms that achieve sub-linear expected regret in $T$. The most common algorithms achieving a sub-linear regret are UCB1 (\cite{auer2002finite}, \cite{lai1985asymptotically}), Thompson Sampling (\cite{russo2018tutorial}), and $\epsilon$-Greedy (\cite{auer2002finite}, \cite{suttonbarto}).

\subsection{Non-stationary MAB Problem} \label{nsmab}

In the non-stationary setting, the rewards $X_t(a)$ for arm $a$ are modeled by a sequence of independent random variables from potentially different distributions that are unknown and \textit{may change} across time. Denote by $\mu_t(a_t)$ the expectation of the reward $X_t(a_t)$ for arm $a_t$ at time step $t$. Similarly, let $a_t^*$ be the arm with the highest expected reward, denoted by $\mu_t^*$, at time $t$. The benchmark performance of the algorithm would be achieved by pulling the optimal arm $a_t^*$ at every time step $t$. The corresponding regret $R_T$ is defined as:
\begin{align}\label{eq-reg1}
R_T = \sum_{t=1}^T \left( \mu^*_t - X_t(a_t) \right).
\end{align} 
The goal is to design algorithms that achieve sub-linear expected regret in $T$ too. In this paper, we consider two commonly-used models of non-stationarity: (i) abruptly changing environment (\cite{ducb}, \cite{gm}) and (ii) continuously changing (\cite{besbes}) environment. The algorithms discussed below for both environments achieve sub-linear expected regret.

\paragraph{Abruptly Changing Environment}
The reward distributions remain fixed during certain periods and change at unknown time instants called breakpoints. Denote by $\beta_T$ the total number of breakpoints that occur before time $T$. As shown in \cite{hartland2006multi}, standard bandit algorithms are not appropriate for this environment, and therefore several methods have been proposed. The most relevant to this paper are the two extensions of the UCB (\cite{auer2002finite}) algorithms: Discounted UCB (DUCB) \cite{ducb} and Sliding Window UCB (SWUCB)  \cite{gm}.

The DUCB algorithm (\textbf{Algorithm \ref{alg-ducb}}) uses a discount factor to emphasize recent rewards when calculating their average. Specifically, the algorithm uses a discount factor $\gamma \in (0,1)$ to calculate the average of the observed rewards:
\begin{equation}
    \bar{X}_t(\gamma, a) = \frac{1}{N_t(\gamma, a)} \sum_{\alpha=1}^t \gamma^{t-\alpha} \mathbf{1}\left(a_{\alpha}=a\right) X_{\alpha}(a).
\end{equation}
Here, $N_t(\gamma, a)$ is the discounted frequency of arm $a$ until time $t$. The algorithm further constructs an upper confidence bound $\bar{X}_t(\gamma, a) + c_t(\gamma, a)$ on the average reward with 
\begin{equation}
c_t(\gamma, a) = 2 \sqrt{\xi \log n_t(\gamma)/N_t(\gamma, a)}
\end{equation}
as the discounted confidence radius (for some constant $\xi$ tuned based on the context; see \cite{gm} for more details). Note that $n_t(\gamma) = \sum_{i=1}^{K} N_t(\gamma, a)$ is the sum of the discounted frequencies for all arms until time $t$. Notice that for $\gamma = 1$, DUCB recovers the UCB1 algorithm.
\begin{algorithm}[]
\begin{algorithmic}[1]
\STATE for $t$ from 1 to $K$, pull arm $a_t=t$; 
\STATE for $t$ from $K + 1$ to $T$, pull arm $a_t$ which maximizes the upper confidence bound $\bar{X}_t(\gamma, a_t) + c_t(\gamma, a_t)$ 
 \caption{Discounted UCB}
 \label{alg-ducb} 
\end{algorithmic}
\end{algorithm}

With the SWUCB algorithm (\textbf{Algorithm \ref{alg-swucb}}), instead of averaging rewards over the entire history with a discount factor, averages are computed based on a fixed-size horizon. At each time step $t$, SWUCB utilizes a local empirical average of the most recent $\tau$ arm pulls to construct an upper confidence bound $\bar{X}_t(\tau, a) + c_t(\tau, a)$ for the expected reward. The local empirical average is defined as:
\[
\bar{X}_t(\tau, a) = \frac{1}{N_t(\tau, a)} \sum_{\alpha=t-\tau-1}^t \mathbf{1}\left(a_{\alpha}=a\right) X_{\alpha}(a)
\]
with $N_t(\tau, a)$ the frequency of selecting arm $a$ in the last $\tau$ arm pulls. The confidence radius is defined as:
\[
c_t(\tau, a) = \sqrt{\xi \log(\min(t, \tau))/N_t(\tau, a)}
\]
with some constant $\xi$ (see \cite{gm} for more details). Notably, in \cite{gm} the authors have demonstrated that both DUCB and SWUCB achieve a regret of $\tilde{O}(\sqrt{\beta_T T})$, where $\tilde{O}(\cdot)$ disregards logarithmic terms.

\begin{algorithm}[]
\begin{algorithmic}[1]
\STATE for $t$ from 1 to $K$, pull arm $a_t=t$; 
\STATE for $t$ from $K + 1$ to $T$, pull arm $a_t$ which maximizes the upper confidence bound $\bar{X}_t(\tau, a_t) + c_t(\tau, a_t)$ 
 \caption{Sliding Window UCB}
 \label{alg-swucb} 
\end{algorithmic}
\end{algorithm}

\paragraph{Continuously Changing Environment}  Here the number of changes in the mean rewards can potentially be infinite, but the total variation over a relevant time horizon is bounded by a \textit{variation budget}; see,  e.g., \cite{besbes}. Specifically, for a time horizon of $T$, we define the variation budget $V_T$ as a non-decreasing sequence of positive numbers {$\{V_t\}_{t=1}^T$} such that $V_1 = 0$ and $KV_t \leq t$, with $K$ the number of arms.

Recall that $\mu_t(a)$ is the expected regret of arm $a$ at time $t$. Denote by $\mu(a)=\{\mu_t(a)\}_{t=1}^T$ the sequence of expected rewards of arm $a$, and $\mu = \{\mu(a)\}_{a=1}^K$ the sequence of expected rewards of all  $K$ arms. 
The set $\mathcal{V}$ of permissible reward sequences for each arm can be written as: 
\begin{align} \label{var}
    \mathcal{V} = \left\{ \mu \in [0,1]^{K \times T} : \sum_{t=1}^{T} \sup_{a \in [1,K] } |\mu_t(a) - \mu_{t+1}(a)| \leq V_T \right\}.
\end{align}
The above set of permissible reward sequences 
can capture various scenarios where the expected rewards may change continuously, in discrete shocks, or adhere to a certain rate of change. 

For different permissible reward sequences, the achievable regrets may be different. We consider their supremum:
\begin{equation}
R_{T}^{\mathcal{V}} = \sup_{\mu \in \mathcal{V}} \left\{\sum_{t=1}^T \mu_t^* - \mathbb{E} \left[ \sum_{t=1}^T X_t(a_t) \right] \right\}.
\end{equation}
In \cite{besbes} the authors provided a near-optimal algorithm with a worst-case regret of $O \left( V^{1/3} T^{2/3} \right)$.

\subsection{Incentivized Exploration} \label{inc-prereq}
As mentioned in Section~\ref{sec:intro}, the principal and the agents may have different interests in many real-world scenarios. The principal would like the agents to select the arms in such a way as to adequately explore different arms to maximize the accumulated rewards. An agent, however, influenced by the feedback of others, behaves myopically in the face of uncertainty, i.e., pulls the arm with the currently highest empirical reward (exploitation only).

Similar to \cite{liu20}, we consider a variant of the MAB problem where a principal aims to incentivize the agents to explore. At each time step {$t$}, an agent pulls one arm $a_t$ based on the recommendation of the principal. The agent receives a reward {$r_t$}, which is then fed back to the principal and the agents. The principal uses a certain bandit algorithm to find the `optimal' arm while balancing exploration and exploitation. When the principal wants to encourage agents to explore, they may offer compensation {$\chi_t$} to the agents. This compensation motivates the agents to follow a specific bandit algorithm that balances exploration and exploitation, ultimately maximizing their cumulative rewards. 

However, because the agent receives compensation, their feedback from pulling the arm can become biased. This bias may introduce a deviation {$\delta_t$} on top of the "true" reward {$X_t (a_t)$}. This deviation is influenced by some unknown, non-decreasing function $f_t$ of the compensation {$\chi_t$}. This function is assumed to possess the following characteristics.
\begin{assumption}
{
(\cite{liu20}) The reward drift function $f_t(x)$ is non-decreasing with $f_t(0)=0$ and is Lipschitz continuous, i.e., there exists a constant $l_t$ such that $|f_t(x)-f_t(y)|\leq l_t |x-y|$ for any $x$ and $y$.
}
\end{assumption}
Note that the received reward $r_t$ is the biased feedback (i.e., equal to the sum $X_t (a_t)+\delta_t$), and the principal and agents cannot distinguish either $X_t (a_t)$ or $\delta_t$ from it. 

Denote by $g_t$ the greedy choice at time $t$, and note that the actual arm pulled, $a_t$, is the arm recommended by the principal. Let $\bar{X}_t(a)$ be the empirical average of the rewards of some arm $a$ until time $t$. Along with the regret, the principal is also concerned with the total compensation he has to pay: 
\begin{align} \label{comp-def}
    C_T = \sum_{t=1}^T \left(\bar{X}_t(g_t) - \bar{X}_t(a_t)\right).
\end{align}

We will characterize the efficacy of incentivized exploration in terms of both \textit{expected} regret and \textit{expected} compensation and aim to answer the following question: if and how can we design algorithms that achieve both sublinear regret and sublinear compensation? The authors in \cite{liu20} have studied this important question in the setting of a \textit{stationary} bandits and proposed algorithms that achieve $O\left(\log T\right)$ regret and compensation. In contrast, in this paper, we investigate the more challenging setting of \textit{non-stationary} bandits.

\section{Incentivized Exploration in Non-Stationary Bandits} \label{ie-sol} In this section, we design algorithms for the incentivized exploration for 
\textit{nonstationary} bandits and show that they achieve sublinear regret and compensation. 

\subsection{Incentivized Exploration in the Abrupty-Changing Environment} \label{ac-sol}
Algorithm \ref{inc-alg} describes a framework of incentivized exploration for the abruptly changing environment. At time $t$, the principal recommends an arm $a_t$ (line 2) based on a non-stationary bandit algorithm (e.g., DUCB or SWUCB), and the greedy choice is denoted by $g_t$ (line 3). The principal offers compensation $\chi_t$ (in line 5) to the agents, which is the difference between the empirical average of rewards from the greedy choice and the principal's recommendation when they differ. This compensation is provided if the principal's recommended arm doesn't align with the greedy choice. After receiving this compensation, the player's outcome is affected by a bias $\delta_t$, which is added to the "true" reward $X_t (a_t)$.

\begin{algorithm}[]
\begin{algorithmic}[1]
 \FOR{$t \in [1,T]$} 
    \STATE $a_t \gets \text{Principal's Recommendation}$
    \STATE $g_t \gets \arg \max_{a \in [1,K]} \bar{X}_t(a)$
    \IF{$a_t \neq g_t$} 
       \STATE $\text{Principal offers compensation of } \chi_t \gets \bar{X}_t(g_t)-\bar{X}_t(a_t) \text{ to the agent.}$ \\
       \STATE $\text{Reward for pulling arm } a_t \text{ is } r_t \gets X_t(a_t) + \delta_t \text{ where reward drift } \delta_t \gets f(\chi_t)$\\
    \ELSE 
        \STATE $r_t \gets X_t(a_t) \text{ is the reward with no compensation}$.\\
     \ENDIF
    \ENDFOR
 
 \caption{Incentivized MAB under Reward Drift}
 \label{inc-alg}
 \end{algorithmic}
\end{algorithm}

With equation (\ref{eq-reg1}), the expected regret is defined as
\begin{align} \label{regret-def}
    \mathbb{E}\left[ \sum_{t=1}^T \left( \mu^* - X_t(a_t) \right) \right] = \sum_{a \neq a^*} \left( \mu^* - \mu(a) \right) \mathbb{E}\left[N_T(a)\right].
\end{align}
Since the expected reward of an arm is in the range $[0,1]$, we have $\left( \mu^* - \mu(a) \right) \leq 1$ for all $a$. Therefore, bounding the expected regret after $T$ pulls essentially amounts to controlling the expected number of times a sub-optimal arm is pulled. 
In Theorem \ref{thm-alg-ducb}, we bound the expected number of times some sub-optimal arms $a \neq a^*_t$ are pulled when the principal uses DUCB algorithm to balance exploration and exploitation in Algorithm \ref{inc-alg} (line 2) until time $T$.

\begin{theorem}[Algorithm \ref{inc-alg} + DUCB Regret Bound]\label{thm-alg-ducb}
{Given the time horizon $T$ and the number of breakpoints $\beta_T$, the expected number of times some sub-optimal arms $a \neq a^*_t$ are pulled is bounded as follows: 
\begin{equation}
    \mathbb{E}\left[N_T(a)\right] \leq \tilde{\eta} \cdot \sqrt{T\beta_T} \log(T)
\end{equation}
}
with some constant $\tilde{\eta} > 0$. 
\label{reg_alg1_ducb}
\end{theorem}

See the proof of Theorem~\ref{thm-alg-ducb} in the Appendix for the choice of the discount factor $\gamma$. 

The following result is for the case when the principal uses the SWUCB algorithm instead in Algorithm \ref{inc-alg} (line 2). 

\begin{theorem}[Algorithm \ref{inc-alg} + SWUCB Regret Bound] \label{thm-alg-swucb}
{ 
Given the time horizon $T$ and the number of breakpoints $\beta_T$, the expected number of times some sub-optimal arms $a \neq a^*_t$ are pulled is bounded as follows: 
\begin{equation}
    \mathbb{E}\left[N_T(a)\right] \leq \tilde{\eta} \cdot \sqrt{\beta_T T\log(T)}
\end{equation}
}
with some constant $\tilde{\eta} > 0$. 
\label{reg_alg1_swucb}
\end{theorem}
{

See the proof of Theorem~\ref{thm-alg-swucb} in the Appendix for the choice of the sliding window $\tau$. 

\begin{remark}
The lower bound of the regret for an algorithm scheme for the abruptly changing environment is $\Omega(\sqrt{T})$ (see section 4 of \cite{gm}). Therefore, the proposed algorithm scheme is optimal up to some $\log T$ powers, besides the dependence on $\beta_T$.
\end{remark}
}

Now, let us take a look at the total expected compensation. Consistent with the definition (\ref{comp-def}), in the non-stationary setting the total compensation is defined as follows: $C_T = \sum_{t=1}^T \left(\bar{X}_t(\gamma, g_t) - \bar{X}_t(\gamma, a_t)\right)$ when using DUCB with discount factor $\gamma$, and $C_T = \sum_{t=1}^T \left(\bar{X}_t(\tau, g_t) - \bar{X}_t(\tau, a_t)\right)$  when using SWUCB with sliding window  $\tau$. 


\begin{theorem}[Algorithm \ref{inc-alg} + DUCB Compensation]
{Given the time horizon $T$ and the number of breakpoints $\beta_T$, the total expected compensation contributed by the arm $a$ when the principal uses the DUCB algorithm is bounded as follows:  
\begin{equation}\label{comp_alg1_ducb}
    \mathbb{E}\left[ C_T(a)\right] \leq \eta \cdot \beta_T^{3/2} \sqrt{T} (\log(T))^{3/2}
\end{equation}
with some constant $\eta > 0$.}
\end{theorem}

\begin{theorem}[Algorithm \ref{inc-alg} + SWUCB Compensation]
Given the time horizon $T$ and the number of breakpoints $\beta_T$, the total expected compensation contributed by the arm $a$ when the principal uses the SWUCB algorithm is bounded as follows:  
\begin{equation}
    \mathbb{E}\left[ C_T(a) \right] \leq \eta \cdot (\beta_T)^{7/4} T^{1/4} (\log(T))^{3/4}
\end{equation}
with some constant $\eta > 0$.
\label{comp_alg1_swucb}
\end{theorem}
The proofs of the above theorems can be found in the Appendix. 

\subsection{Incentivized Exploration in the Continuously-Changing Environment} \label{cc-sol}
For the continuously changing environment, we divide the time horizon $T$ into $\lfloor T/\sigma \rfloor$ batches of certain fixed size $\sigma$. At the start of each batch, the principal restarts a certain bandit algorithm to recommend arms to the agent.   
\begin{algorithm}[]
\begin{algorithmic}[1]
\REQUIRE $\sigma, K, T$
\ENSURE $j = 1$
\WHILE{$j \leq \big \lceil{T/\sigma \big \rceil}$} 
    \STATE $\alpha \gets (j-1)\sigma$\\
    \FOR{$t = 1,.., \min\{T, \alpha + \sigma\}$} 
        \STATE $a_t \gets \text{Principal's recommendation}$\\
        \STATE $g_t \gets \arg\max_{a \in [1,K]} \bar{X}_t(a)$\\
        \STATE $\text{steps 4-9 from \textbf{Algorithm \ref{inc-alg}}}$\\
    \ENDFOR
    \STATE $\text{Increment } j \gets j + 1 \text{ and return at line 1.}$\\
\ENDWHILE
 \caption{Restarting technique with a MAB algorithm}
 \label{alg-cce}
 \end{algorithmic}
\end{algorithm}
We present Algorithm \ref{alg-cce} where for a single batch the principal employs certain MAB algorithm such as UCB1, $\epsilon$-Greedy or Thompson Sampling and recommends an arm to the agent (line 4) to balance exploration and exploitation. Based on the greedy choice of the agent (line 5), the compensation scheme is determined according to lines 4-9 of Algorithm \ref{inc-alg}. This process is repeated for $\lfloor T/\sigma \rfloor$ iterations. 

\begin{theorem} [Algorithm \ref{alg-cce} Regret]
Given the time horizon $T$ and the variation budget
$V_T$, if the principal employs \textit{UCB1}, $\epsilon$-greedy or Thompson Sampling for recommending arms to the agent in Algorithm \ref{alg-cce}, then the worst-case regret over the time horizon $T$ is bounded as follows: 
\begin{align}
    R_{\mathcal{V}}^T \leq \eta\cdot V_T^{1/3} \left(K \log(T) \right)^{1/3} T^{2/3}
\end{align}
with some constant $\eta > 0$.
\label{reg_cce}
\end{theorem}

See the proof of Theorem~\ref{reg_cce} in the Appendix for the choice of the batch size $\sigma$. 

The overall compensation $C^T_{\mathcal{V}}$ for the entire time horizon $T$ is calculated by summing up the compensation of each batch, following the definition (\ref{comp-def}). 

\begin{theorem}
Given the time horizon $T$ and the variation budget
$V_T$, if the principal employs \textit{UCB1} for recommending arms to agents in Algorithm \ref{alg-cce}, then the worst-case total compensation is bounded as follows: 
\begin{align}
    C_{\mathcal{V}}^T  \leq \eta_1 \cdot \left( KV_T \log(T)\right)^{1/3} T^{2/3}
\end{align}
with some constant $\eta_1 > 0$. If the principal employs $\epsilon$-greedy, then the worst-case total compensation is bounded as follows: 
\begin{align}
    C_{\mathcal{V}}^T  \leq \eta_2 \cdot \left(K V_T \log(T)\right)^{2/3} T^{1/3}
\end{align}
with some constant $\eta_2 > 0$. If the principal employs Thompson Sampling, then the worst-case total compensation is bounded as follows:
\begin{align}
    C_{\mathcal{V}}^T  \leq \eta_3 \cdot \left(K V_T \log(T)\right)^{2/3} T^{1/3}
\end{align}
with some constant $\eta_3 > 0$.

\label{comp_cce}
\end{theorem}

See the Appendix for the proof of the above theorem. 
    


\section{Numerical Experiments}
In this section, we evaluate the performance of the proposed algorithms in Section~\ref{ie-sol} numerically.  

\subsection{Abruptly Changing Environment} \label{sim-ac}

\begin{figure}
    \centering
    \includegraphics[width = \linewidth]{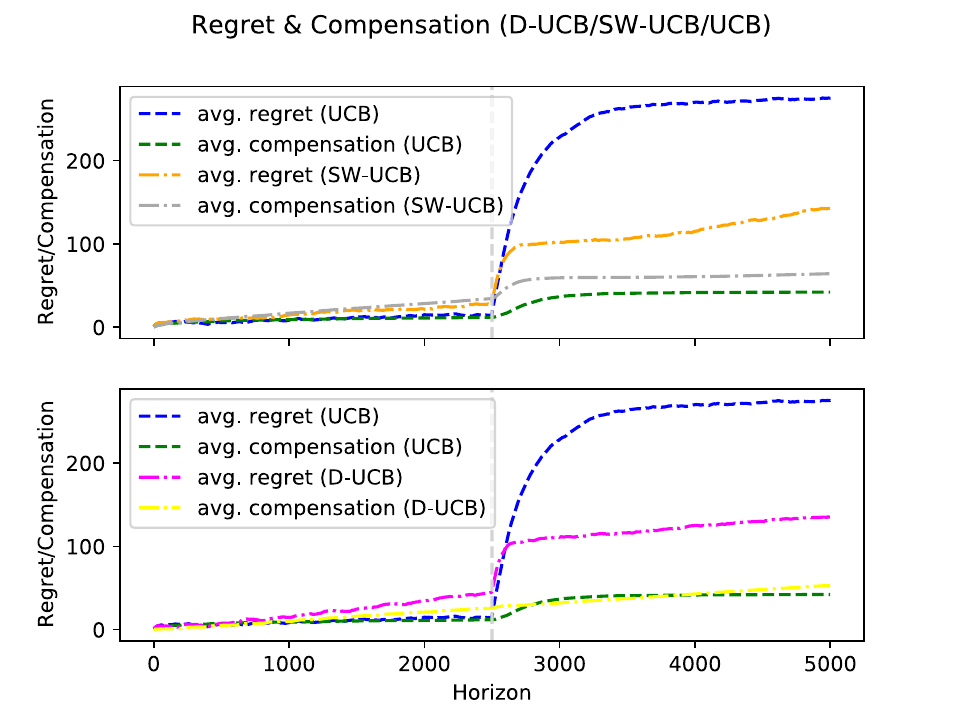}
    \caption{\footnotesize (Upper) Regret and Compensation performance of DUCB with Algorithm \ref{inc-alg} with $\gamma_C = 10$ (Below) Regret and Compensation performance of SWUCB with Algorithm \ref{inc-alg} with $\tau_C = 0.9$, both with $T=5000$ and $\beta_T = 1$}
    \label{fig:piecewise-br-1}
\end{figure}
\begin{figure}
    \centering
    \includegraphics[width = \linewidth]{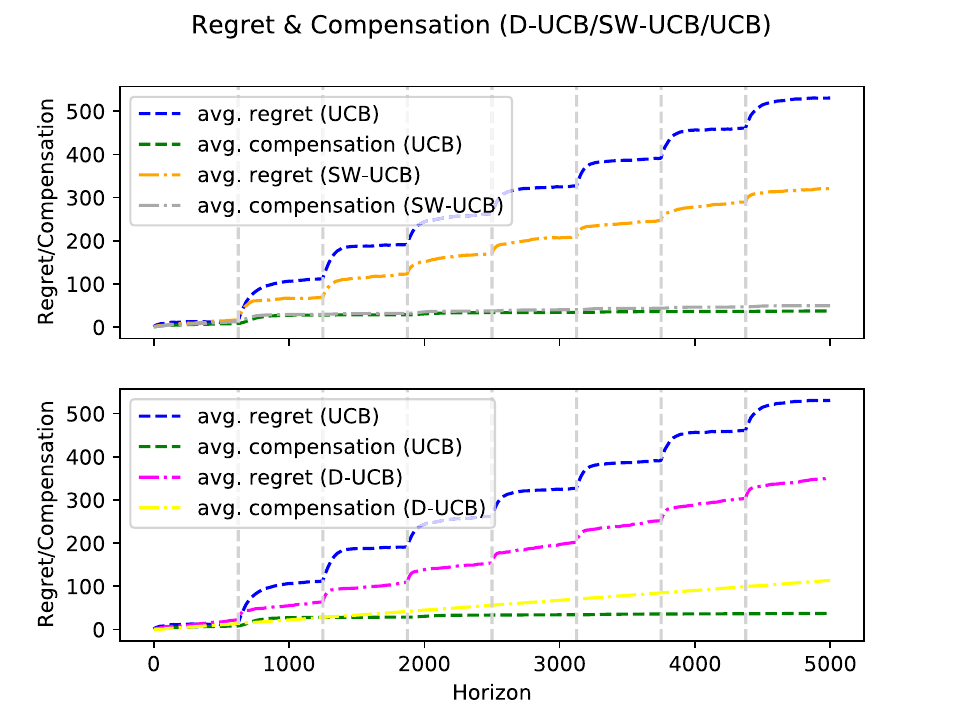}
    \caption{\footnotesize (Upper) Regret and Compensation performance of DUCB with Algorithm \ref{inc-alg} with $\gamma_C = 40$ (Below) Regret and Compensation performance of SWUCB with Algorithm \ref{inc-alg} with $\tau_C = 1$, both with $T=5000$ and $\beta_T = 1$}
    \label{fig:piecewise-br-7}
\end{figure}

Consider a setting of two arms (i.e., $K=2$; indexed by 1 and 2), with the expected rewards 0.99 and 0.01, respectively. The expected rewards flip (i.e., swaps values) at every breakpoint that is evenly located at {$kT/p$} for $k= 1, 2, \cdots, p - 1$ for some integer  $p > 0$. 
For instance, if  $p=3$, we have the breakpoints at {$\lfloor{T/3 \rfloor}$} and {$\lfloor{2T/3 \rfloor}$}.

In Algorithm \ref{inc-alg}, the discount factor $\gamma = 1 - (1/\gamma_C)\sqrt{\beta_T / T}$ for DUCB is chosen according to the analysis of Theorem \ref{reg_alg1_ducb} and the window size $\tau = \lfloor{\tau_C \sqrt{T \log(T)/ \beta_T} \rfloor}$ for SWUC according to Theorem \ref{reg_alg1_swucb}, in order to minimize the regret. 
We run a hundred repetitions to obtain average values of regret and compensation. 

Figures \ref{fig:piecewise-br-1} and \ref{fig:piecewise-br-7} present the performance of Algorithm \ref{inc-alg} with DUCB and SWUCB as submodules with $T=5000$. Both algorithms outperform the UCB1 with Algorithm \ref{inc-alg}. The frequent changes force the UCB1 to make mistakes at the start of each breakpoint, as it considers the \textit{entire history}. However, DUCB considers the decaying history, giving more importance to the recent past; and SWUCB considers a sliding window and thus adjusts quickly to changes in the reward distribution,  leading to lower regret.

Table \ref{tab:piecewise-tab} summarizes the performance of SWUCB and DUCB with Algorithm \ref{inc-alg}, respectively, with varying breakpoints. We present the corresponding parameters { $\tau_C$} and { $\gamma_C$}, which minimize the regrets. All the regret and compensation values are within the theoretical bounds. Besides, the regret is consistently lower than the UCB1 counterpart for DUCB and SWUCB, as all the parameters $\gamma, \gamma_C, \tau, \tau_C$ are tuned to minimize regret. The values considered are $\tau_C = [10, 20, 30, 40]$ and $\gamma_C = [0.9, 0.95, 1, 2]$ through experimentation.  

In both cases the regret is growing in the order of $O\left(\sqrt{\beta_T}\right)$ with varying breakpoints, as the theoretical analysis suggests. As for compensation, SWUCB increases more rapidly than DUCB, which explains the higher sensitivity of SWUCB to the number of breakpoints, which is in the order of $O\left(\beta_T^{7/4}\right)$ compared to $O\left(\beta_T^{3/2}\right)$. 

\begin{flushleft}
    \begin{table}
    \centering
    \caption{{\scriptsize  The performance of Algorithm \ref{inc-alg} with varying number of breakpoints {$\beta_T$}. The subscripts $U, D, S$ stand for UCB1, DUCB, and SWUCB, respectively, with $R$ as the regret and $C$ as the compensation values.}}
    \begin{tabular}{@{}lllllllll@{}}
    \toprule
    {\footnotesize $\beta_T$} & {\footnotesize $\gamma_C$} & {\footnotesize $\tau_C$} & {\footnotesize $R_U$} & {\footnotesize $R_S$} & {\footnotesize $R_D$} & {\footnotesize $C_U$} & {\footnotesize $C_S$} & {\footnotesize $C_D$}\\
    \midrule
    1 & 15	& 	1	& 	275.2	& 	135.1 & 142.7 & 42.1 & 53.2	& 64.2\\
    2 & 10 & 1 & 364.2 & 203.5 & 205.7 & 42.5 & 70.7 &  92.3 \\
    3	& 	15	& 	1	& 	430.4	& 	239.5		& 247.1	& 	41.6	& 81.2	& 82.8\\
    4 & 15 & 0.95 & 394.8 & 264.1 & 259.7 & 41.4 & 95.1 & 89.2\\
    5	& 	10	& 	1	& 	423.7	& 288.9	& 	302.4	& 39.8	& 	100.8	& 112.3\\
    6	& 25	& 1 & 	481.8 &	330.1 &	279.1 &	38.5 &	107.9 &	67.6\\
    7 & 30 & 0.95 & 484.2 & 339.0 & 299.7 & 38.6 & 117.1 & 59.2\\
    \bottomrule
    \end{tabular}
    
    \label{tab:piecewise-tab}
\end{table}
\end{flushleft}

\subsection{Continuously Changing Environment} \label{sim-cc}
Consider a setting of two arms, indexed by 1 and 2. The received (i.e., instantaneous) reward $X_t(a_t)$ for arm $a_t$ at time $t$ is modeled as a Bernoulli random variable with a changing expectation $\mu_t(a_t)$
\begin{align} \label{sin}
X_t(a_t) = 
     \begin{cases}
       1 &\quad\text{with prob. } \mu_t(a_t)\\
        0 &\quad\text{with prob. } 1 - \mu_t(a_t)\\ 
     \end{cases}
\end{align}   
for all $t \leq T$ and for any pulled arm $a_t \in [1,K]$. We have two sinusoidal evolution patterns (Fig. \ref{fig:Algorithm 2}; inspired by \cite{besbes}) with a variation budget $V_T = 3$ for $\mu_t(a_t)$ (the green dotted line is for arm 1 as $\mu_1$ and the yellow dotted line is for arm 2 as $\mu_2$ in Fig. \ref{fig:Algorithm 2}). They describe different changing environments under the same variation budget. In the first (i.e., left half) instance, the variation budget is spent throughout the whole horizon, while in the second, the same variation budget is spent only over the first third of the whole horizon. In this experiment, at each time step $t$, the following happens in order: (i) the algorithm selects an arm $a_t \in [1,K]$, (ii) the binary rewards are generated based on (\ref{sin}), and (iii) the instantaneous reward $X_t(a_t)$ is observed. We denote by $a_t^* = \arg\max_{a \in [1,K]} \mu_{t}(a)$ as the optimal arm at time $t$. The instantaneous regret at time $t$ is $X_t(a_t^*) - X_t(a_t)$. We run the experiment with multiple (in this case 2000 times) repetitions to obtain average values of regret and compensation.

Fig. \ref{fig:Algorithm 2-rc} presents the performance of Algorithm \ref{alg-cce} with UCB1, $\epsilon$-Greedy and Thompson Sampling with regards to \textit{regret} and \textit{compensation}, whereas  Fig. \ref{fig:Algorithm 2} presents the total rewards accumulated (i.e., averaged over 2000 iterations) for the same settings. Table~\ref{tab:cont-tab} presents more data about the same performance with varying degrees of variation budget $V_T$. Note that the regret for all the algorithm schemes varies (i.e., increase) by a maximum of $V_T^{1/3}$ as suggested by the theoretical analysis. compensation increases more quickly for $\epsilon$-Greedy and Thompson Sampling compared to UCB1, as UCB1 is comparatively less sensitive to $V_T$. Notice that $\epsilon$-greedy does not drastically change its compensation values compared to Thompson Sampling, suggesting that the upper bound can be tightened. 

\begin{figure}
    \centering
    \includegraphics[width = \linewidth]{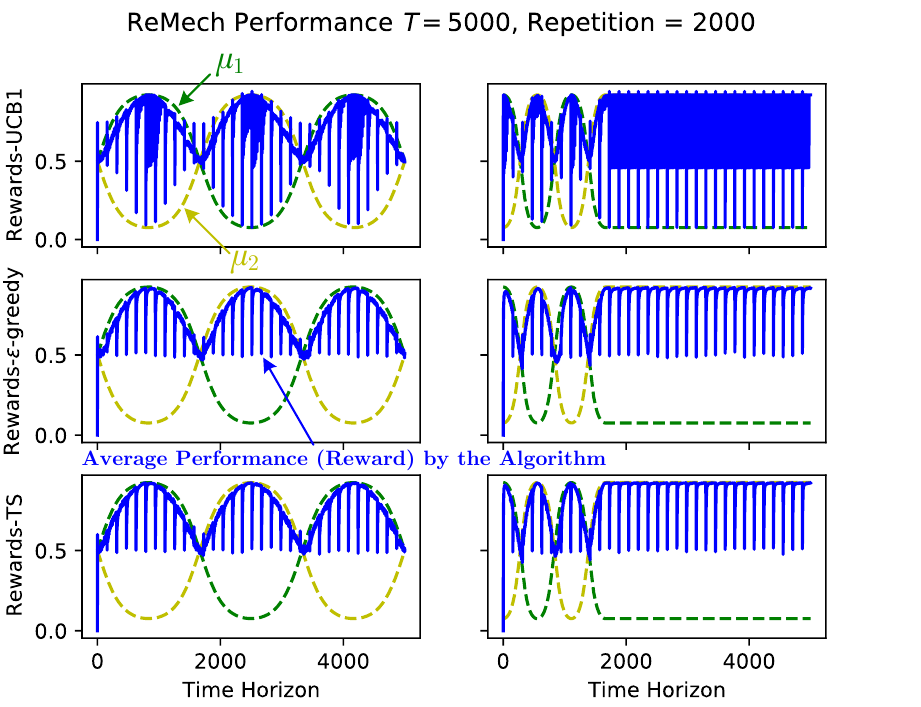}
    \caption{{\footnotesize Algorithm \ref{alg-cce} (written as ReMech, shorthand for restarting mechanism, in the diagram) performance with $T = 5000$ with 2000 repetitions. The blue curve traces the total reward accumulated (averaged over all iterations) with \textsc{Algorithm \ref{alg-cce}} at various time steps with UCB1, $\epsilon$-greedy, and Thompson Sampling as respective submodules.}}
    \label{fig:Algorithm 2}
\end{figure}

\begin{figure}
    \centering
    \includegraphics[width = \linewidth]{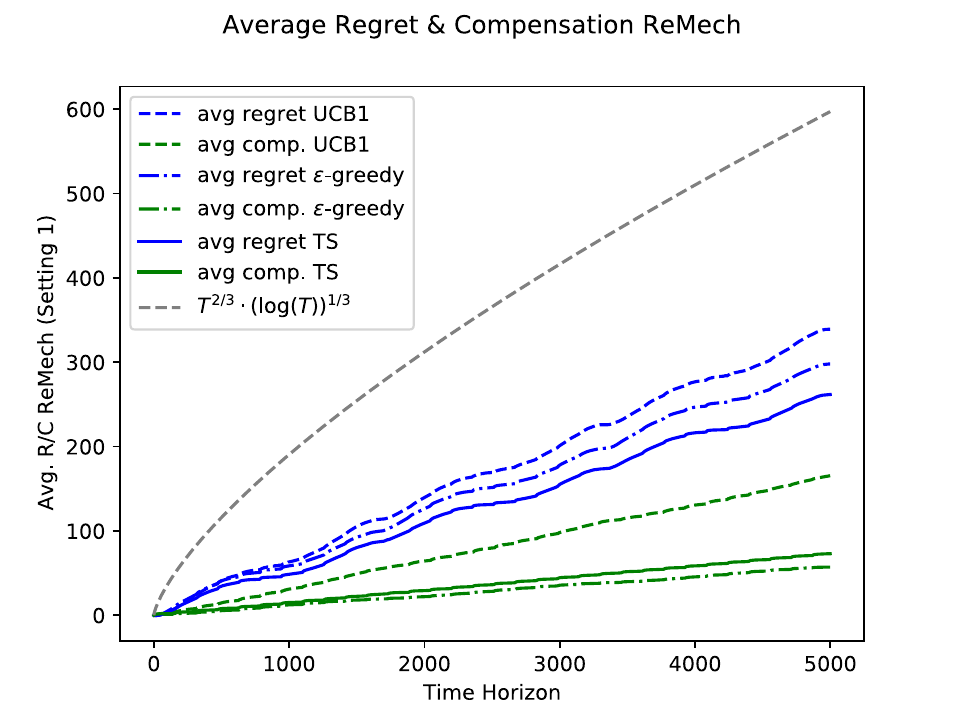}
    \caption{\footnotesize Algorithm \ref{alg-cce} performance with submodules of UCB1, $\epsilon$-greedy, and Thompson Sampling, for a large horizon with $T = 5000$ with $2000$ repetitions.}
    \label{fig:Algorithm 2-rc}
\end{figure}

\begin{flushleft}
    \begin{table}
    \centering
    \caption{{\scriptsize The performance of Algorithm \ref{alg-cce} with different variation budgets. The subscripts $U, \epsilon G, T$ stand for UCB1, $\epsilon$-Greedy and Thompson Sampling, respectively, with $R$ as the regret and $C$ as the compensation values.}}
    \begin{tabular}{@{}llllllll@{}}
    \toprule
    {\footnotesize $V_T$} & 3 & 6 & 9 & 12 & 15 & 18 & 24\\
    \midrule
    {\footnotesize $R_{U}$} & 156.1 & 175.9 & 185.7 & 191.4 & 198.3 & 207.5 & 210.3\\
    {\footnotesize $C_{U}$} & 88.9 & 107.4 & 119.8 & 127.2 & 135.0 & 145.6 & 149.8\\ 
    \midrule
    {\footnotesize $R_{\epsilon G}$} & 143.1 & 164.1 & 180.2 & 192.0 & 202.3 & 215.0 & 229.0\\
    {\footnotesize $C_{\epsilon G}$} & 37.3 & 52.4 & 64.1 & 73.6 & 80.7 & 88.7 & 99.5\\
    \midrule
    {\footnotesize $R_{T}$} & 125.1 & 147.2 & 163.8 & 177.8 & 185.9 & 197.8 & 211.6\\
    {\footnotesize $C_{T}$} & 48.4 & 69.0 & 84.9 & 97.5 & 107.5 & 118.1 & 132.2\\
    \bottomrule
    \end{tabular}
    
    \label{tab:cont-tab}
\end{table}
\end{flushleft}

\section{Conclusion}
We have studied the incentivized exploration for the MAB problem with non-stationary reward distributions, where the players receive compensation for exploring arms other than the greedy choice and may provide biased feedback on the reward. We consider two different environments that capture rewards' non-stationarity, and propose incentivized exploration algorithms accordindingly. We show that the proposed algorithms achieve sublinear regret and compensation in time and thus are effective in incentivizing exploration despite the nonstationarity and the drifted feedback. 


\bibliographystyle{IEEEtran}  
\bibliography{refs} 

\section{Appendix} \label{app}
\subsection{Material for Theorem 1}

In this section we discuss the details pertaining to the incentivized exploration on abruptly changing environment, specifically Theorem 1 (i.e., regret result for Algorithm 3 + DUCB). We start with laying out some fundamental results (i.e., lemma 1 and 2) required for the analysis later. After that, we present lemma 4 which directly leads to the main result.

Recall that $N_t(\gamma, a)$ is the discounted frequency (i.e., number of pulls) of arm $a$ till time $t$ and the sum of this quantity of all the arms is denoted by $n_t(\gamma)$. In the following lemma we upper bound the same.  

\begin{lemma}
\label{lem22}
$n_t(\gamma) \leq \min(t, 1 / (1 - \gamma))$, for all $t \leq T$.
\end{lemma}

\begin{proof} According the definition, we have the following;
\begin{equation} 
\begin{split}
n_t(\gamma) & = \sum_{a=1}^{K} N_t(\gamma, a)\\
&= \sum_{a=1}^{K} \sum_{s=1}^t \mathbf{1}\left(a_{s} = a\right) \gamma^{t - s}\\
&= \sum_{s = 1}^t \gamma^{t - s} \leq \sum_{s=1}^{\infty} \gamma^{s} = \frac{1}{1 - \gamma}.
\end{split}
\end{equation}

Notice that $n_t(\gamma)$ cannot be more than $t$, we get the final upper bound.
\end{proof}

Next, we look at the total discounted reward drift $D_t(\gamma, a)$ contributed by some arm $a$ till some time $t$. Note that $l_t$ is the Lipschitz constant corresponding to the drift function, at any given time $t$ (Refer Assumption 1 in the main paper), and $N_t(a)$ as the total pulls (i.e., essentially $N_t(\gamma,a)$ for $\gamma = 1$) for arm $a$ till time $t$. We define $l = \max_t l_t$ for the following lemma. 

\begin{lemma}
$D_t(\gamma, a) \leq 2l N_t(a) \sqrt{\xi \log(n_t(\gamma))}$ for all arms $a \in [1,K]$ and $t \in [1,T]$.
\end{lemma}

\begin{proof}

Notice that there is a drift in the reward from the agent only when the principal provides some compensation. The inequalities mentioned below quantify this event, and they must hold true simultaneously. Since the agent is assumed to be greedy in nature, the arm $g_t$ (i.e., the arm with the highest empirical average till time $t$) seems more appealing than $a_t$ (i.e., the principal's recommended arm at time $t$). The equality below represent this event.

\begin{equation}
    \bar{X}_{t}(\gamma, g_t) \geq \bar{X}_{t}(\gamma, a_t),
\end{equation}

However, the principal recommends arm $a_t$ as per it's algorithm which balances exploration and exploitation. The inequality below quantifies this event.

\begin{equation}
    \bar{X}_{t}(\gamma, g_t) + c_t(\gamma, g_t) \leq \bar{X}_{t}(\gamma, a_t) + c_t(\gamma, a_t).
\end{equation}

The above two inequalities imply {  $\bar{X}_{t}(\gamma, g_t) - \bar{X}_{t}(\gamma, a_t) \leq  c_t(\gamma, a_t) -  c_t(\gamma, g_t)$}. Since the drift function is 
Lipschitz smooth (Assumption 1 in the main paper) we can upper bound the instantaneous drift $\delta_t(a_t)$ for the pulled arm $a_t$ at time $t$ as below,

\begin{equation} 
\begin{split}
\delta_t(a_t) & \leq l_t \left(c_t(\gamma, a_t) -  c_t(\gamma, g_t)\right)\\
& \leq l_t \left ( 2 \sqrt{\frac{\xi \log(n_t(\gamma))}{N_t(\gamma, a_t)}} \right) = 2l_t \sqrt{\frac{\xi \log(n_t(\gamma))}{N_t(\gamma, a_t)}}.
\end{split}
\end{equation}

Therefore, summing the instantaneous drift till some time $t$, for some arm $a$ gives us its total drift contribution till said $t$, as shown below.


\begin{equation} 
 \begin{split}
        D_t(\gamma, a) & = \sum_{s=1}^t \mathbf{1}\left(a_{s} = a \right) \delta_{s}(a)\\
        &= \sum_{s=1}^t \mathbf{1}\left(a_{s} = a \right) 2l_{s} \sqrt{\frac{\xi \log(n_t(\gamma))}{N_{s}(\gamma, a)}}\\
        & \leq 2l \sqrt{\xi \log(n_t(\gamma))} \sum_{s=1}^t \mathbf{1}\left( a_{s} = a \right)\frac{1}{\sqrt{N_{s}(\gamma, a)}} \\
        & \leq 2l \sqrt{\xi (1-\gamma) \log(n_t(\gamma))} \sum_{s=1}^{N_t(a)}1 \\
        & \leq 2lN_t(a) \sqrt{\xi \log(n_t(\gamma))}
    \end{split}
\end{equation}

\end{proof}

 Recall from section 2-A that since the expected reward of an arm is in the range $[0,1]$, we get $\left( \mu^* - \mu(a) \right) \leq 1$ for all arms $a \in [1,K]$. Therefore, bounding the expected regret after $T$ pulls essentially amounts to controlling the expected number of times a sub-optimal arm is pulled. In the next lemma, we bound that expected number of total pulls $N_T(a)$ of some sub-optimal arm $a \neq a_t^*$ for the entire horizon $t \leq T$.

\begin{lemma}
Let $\xi > 1/2, \ T > 1$ and $\gamma \in (0,1)$. For any arm $a \in [1,K]$ and total number of breakpoints $\beta_T$, we have
\begin{equation}
    \mathbb{E} \left[ N_T(a) \right] \leq \left(B(\gamma)T(1-\gamma) + C(\gamma)\frac{\beta_T}{1-\gamma}\right)\log\left(\frac{1}{1-\gamma}\right)
\end{equation}
where
\begin{multline}
    B(\gamma) = \frac{16(1-\gamma)\xi}{F(\gamma)} \frac{\lceil{T(1-\gamma)\rceil}}{T(1-\gamma)} \\ + \frac{2\Big \lceil{-\log(1-\gamma)/\log(1+4\sqrt{1-1/2\xi)} \Big \rceil}}{-\log(1-\gamma)(1-\gamma^{1/(1-\gamma)})}
\end{multline}
and
\begin{equation}
    F(\gamma) = \gamma^{1/(1-\gamma)}\left(\Delta_{\mu_T}(i)\sqrt{1-\gamma} - 4l\sqrt{-\xi \log(1-\gamma)}\right)^2
\end{equation} and
\begin{equation}
    C(\gamma) = \frac{(\gamma-1)\log((1-\gamma)\xi\log(n_K(\gamma))}{\log(1-\gamma)\log(\gamma)}
\end{equation}
\label{reg_ducb_thm}
\end{lemma}

\begin{proof} 
Let us denote by { $\Delta_{\mu_T}(a_t)$} the minimum of the difference of the expected reward $\mu_t(a^*_t)$ of the best arm at time $t$ and the expected reward $\mu_t(a_t)$ of the arm $a_t$ for all $t \in [1,T]$ when $a_t$ is not the optimal arm, i.e.,
\begin{equation}
\Delta_{\mu_T}(a_t) = \min_{t \in [1,T]; \ a_t \neq a^{*}_t} \left(\mu_t(a^{*}_t) - \mu_t(a_t)\right).    
\end{equation}
The quantity $N_T(a)$ for some arm $a$ can be represented as the following.

\begin{equation} 
\begin{split}
N_T(a) = \sum_{t=1}^T \mathbf{1}\left(a_t = a \neq a^{*}_t\right)
& = 1 + \sum_{t=K+1}^T \left(a_t = a \neq a^{*}_t\right)\\
& = 1 + H + G
\end{split}
\end{equation}

where $H = \sum_{t=K+1}^T \mathbf{1}\left(a_t=a \neq a_{t}^*;N_t(\gamma,a) < A(\gamma)\right)$ and $G = \sum_{t=K+1}^T \mathbf{1}\left(a_t=a \neq a_{t}^*;N_t(\gamma,a) \geq A(\gamma)\right)$ and where

\begin{equation}
    A(\gamma) = \frac{16(1 - \gamma)\xi \log(n_t(\gamma))}{\left( \Delta_{\mu_T}\sqrt{1-\gamma}-4l\sqrt{(1 - \gamma)\xi \log(n_t(\gamma))}\right)^2}.
\end{equation}

The next few steps directly follow from the analysis in \cite{gm}. For the same definitions of $D(\gamma)$ and $\mathcal{T}(\gamma)$ from \cite{gm}, we can bound $N_T(a)$ by:
\begin{multline}
 N_T(a) \leq 1 + \lceil{T(1 - \gamma)}\rceil A(\gamma)\gamma^{-1/(1 - \gamma)} + \beta_T D(\gamma) + \\ \sum_{t \in \mathcal{T}(\gamma)} \mathbf{1}\left(a_t=a \neq a_{t}^*;N_t(\gamma,a) \geq A(\gamma)\right)   
\end{multline}

Now, for some time $t$ the event \{$G : \mathbf{1}\left(a_t=a \neq a_{t}^*;N_t(\gamma,a) \geq A(\gamma)\right)$\} will occur when the algorithm has the following inequality true, for the arm $a \neq a^*_t$ which is labelled as the event $\mathcal{Z}$;

\begin{equation}
    \mathcal{Z} : \bar{X}_{t}(\gamma, a) + c_t(\gamma, a) > \bar{X}_{t}(\gamma, a^{*}_t) + c_t(\gamma, a^{*}_t).
\end{equation}

Let us denote by $\bar{Y}_t(\gamma, a)$ the discounted empirical average of arm $a$ \textit{excluding} the drift at each time till $t$ (i.e., pure empirical average, different from $\bar{X}_t(\gamma, a)$). We have the following inequalities from $\mathcal{Z}$;

\begin{equation}
    \bar{Y}_{t}(\gamma, a) + \frac{D_t(\gamma, a)}{N_t(a)} + c_t(\gamma, a) > \bar{Y}_{t}(\gamma, a^{*}_t) + \frac{D_{t}(\gamma, a^{*}_t)}{N_t(a^{*}_t)} + c_t(\gamma, a^{*}_t).
\end{equation}

Therefore, the upper bound for the difference between the expected reward of the optimal arm $a^{*}_t$ and the current arm $a$ is
\begin{equation} 
\begin{split}
\bar{Y}_{t}(\gamma, a^{*}_t) - \bar{Y}_t(\gamma,a) & < \frac{D_t(\gamma, a)}{N_t(a)} + c_t(\gamma, a) \\
& = 2\sqrt{\xi\log(n_t(\gamma))}\left(l + \frac{1}{\sqrt{N_t(\gamma, a)}}\right).\\
\end{split}
\end{equation}

Now, we can decompose the event $G$ as the following: For $\mathcal{Z}$ to happen, at least one of the events $G_t^i$ has to occur:
\begin{equation}
    \therefore \mathbf{1}\left(a_t=a \neq a_{t}^*;N_t(\gamma,a) \geq A(\gamma)\right) \subseteq G_t^1 \cup G_t^2 \cup G_t^3,
\end{equation}

where 
\begin{equation}
    G_t^1 = \{\bar{X}_t(\gamma,a) > \mu_t(\gamma,a) + c_t(\gamma, a) \},
\end{equation}
\begin{equation}
    G_t^2 = \{ \bar{X}_t(\gamma,a_{t}^*) < \mu_t(\gamma,a_{t}^*) - c_t(\gamma, a_{t}^*) \},
\end{equation}
and 
\begin{multline}
    G_t^3 =  \hat{\mu}_{t}(\gamma, a_{t}^*) - \hat{\mu}_t(\gamma,a) \\ < 2\sqrt{\xi\log(n_t(\gamma))}\left(l + \frac{1}{\sqrt{N_t(\gamma, a)}}\right) 
\end{multline}

The event $G_t^1$ occurs when the algorithm \textit{overestimates} the average reward the pulled arm $a$, the event $G_t^2$ occurs when the algorithm \textit{underestimates} the average reward of the best arm $a^*$, and $G_t^3$ occurs when the difference between expected rewards for both the arms $a$ and $a^{*}_t$ are too small.   

Note that for the choice of $A(\gamma)$ and the definition of $\Delta_{\mu_T}(a)$, the event $G_t^3$ never occurs, as

\begin{equation*}
    \begin{split}
        \bar{Y}_{t}(\gamma, a^*) - \bar{Y}_t(\gamma,a)  & <  2\sqrt{\xi\log(n_t(\gamma))}\left(l + \frac{1}{\sqrt{N_t(\gamma, a)}}\right) \\
        & \leq 2\sqrt{\xi\log(n_t(\gamma))}\left(l + \frac{1}{\sqrt{A(\gamma)}}\right) \\
        & = \frac{\Delta_{\mu_T}(a)}{2}
    \end{split}
\end{equation*}

Therefore, we have the probability of the event $G_t^3$ as $\mathbb{P}\left[G_t^3\right] = 0$ and from \cite{gm} we get the corresponding probability for $G_t^1$ and $G_t^2$ as

\begin{equation}
    \mathbb{P}\left[ G_t^1 \right] = \mathbb{P}\left[ G_t^2 \right] \leq \Bigg \lceil{\frac{\log(n_t(\gamma))}{\log(1 + \eta)}} \Bigg \rceil n_t(\gamma)^{-2\xi \left(1 - \frac{\eta^2}{16}\right)}
\end{equation}

We have the final bound, by taking $\xi > 1/2$ and $\eta = 4\sqrt{1 - (1/2\xi)}$, so as to make $2\xi(1 = \eta^2/16) = 1$:

$$\mathbb{E}\left[ N_T(a) \right] \leq 1 + \lceil{T(1 - \gamma)}\rceil A(\gamma)\gamma^{-1/(1 - \gamma)} + \beta_T D(\gamma) + Y$$

where
\begin{equation}
    Y = \frac{1}{1 - \gamma} + \Bigg \lceil{\frac{\log\left(\frac{1}{1 - \gamma}\right)}{\log(1 + 4\sqrt{1 - (1/2\xi)})} \Bigg \rceil} \frac{T(1 - \gamma)}{1 - \gamma^{\left(1/(1 - \gamma)\right)}}
\end{equation}

We obtain the statement of the lemma by substituting the values of {  $A(\gamma)$}, {  $D(\gamma)$}, and {  $n_t(\gamma)$}.
\end{proof}

Note that lemma 4 has done most of the heavy lifting, and just choosing an appropriate discount factor (i.e., $\gamma = 1 - \eta \cdot\sqrt{\beta_T/{T}}$) to minimize the expression gives us the result for Theorem 1.

\subsection{Material for Theorem 2}
In this section, we discuss the details pertaining to the incentivized exploration of an abruptly changing environment, specifically Theorem 2 (i.e., the regret result for Algorithm 3 + SWUCB). We start with laying out a fundamental result (i.e., Lemma 4) required for the analysis of Lemma 5, which directly leads to the main result.

To begin, we look at the total reward drift $D_t(\tau, a)$ contributed by some arm $a$ until some time $t$ with a sliding window of size $\tau$. As mentioned in the analysis of Theorem 1, $l_t$ is the Lipschitz constant corresponding to the drift function at any given time $t$ (Refer Assumption 1 in the main paper), and $N_t(a)$ as the total pulls (i.e., essentially $N_t(\tau,a)$ for $\tau = t$) for arm $a$ until time $t$. We define $l = \max_t l_t$ for the following Lemma. 

\begin{lemma}
{ 
\label{lem24}
$D_t(\tau, a) \leq l N_t(a) \sqrt{\xi \log(\min(t,\tau))}$, for all arms $a \in [1,K]$ and $t \in [1,T]$.
}
\end{lemma}

\begin{proof}
The reasoning is similar to the analysis of Lemma 2. There is a drift in the reward from the agent only when the principal provides some compensation. The inequalities mentioned below, quantifying this event, must hold true simultaneously. Since the agent is assumed to be greedy in nature, the arm $g_t$ (i.e., the arm with the highest empirical average until time $t$) seems more appealing than $a_t$ (i.e., the principal's recommended arm at time $t$). However, the principal recommends arm $a_t$ per its algorithm, balancing exploration and exploitation. The inequalities below represent the same. 
\begin{equation}
    \bar{X}_{t}(\tau, g_t) \geq \bar{X}_{t}(\tau, a_t)
\end{equation}
\begin{equation}
    \bar{X}_{t}(\tau, g_t) + c_t(\tau, g_t) \leq \bar{X}_{t}(\tau, a_t) + c_t(\tau, a_t)
\end{equation}

The two inequalities imply $\bar{X}_{t}(\tau, g_t) - \bar{X}_{t}(\tau, a_t) \leq  c_t(\tau, a_t) -  c_t(\tau, g_t)$. Since the drift function is 
Lipschitz smooth (Assumption 1 in the main paper), we can upper bound the instantaneous drift $\delta_t(a_t)$ for the pulled arm $a_t$ at time $t$ as below,

\begin{equation} 
\begin{split}
\delta_t(a_t) & \leq l_t \left(c_t(\tau, a_t) -  c_t(\tau, g_t)\right)\\
& \leq l_t \sqrt{\frac{\xi \log(\min(t, \tau))}{N_t(\tau, a_t)}} 
\end{split}
\end{equation}

Therefore, summing the instantaneous drift $\delta_t(a)$ until some time $t$, for some arm $a$ gives us its total drift contribution until said $t$, as shown below.

\begin{equation} 
\begin{split}
D_t(\tau, a) & = \sum_{s=1}^t \mathbf{1}\left( a_{s} = a \right)  \delta_s(a)\\
&= \sum_{s=1}^t \mathbf{1}\left( a_{s} = a \right) l_{s} \sqrt{\frac{\xi \log(\min(t, \tau))}{N_{s}(\tau, a)}}\\
& \leq l \sqrt{\xi \log(\min(t, \tau))} \sum_{s^{\prime}=1}^{N_t(a)} \frac{1}{\sqrt{N_{s^{\prime}}(\tau, a)}} \\
& \leq  N_t(a) l \sqrt{\xi \log(\min(t, \tau))} \ \  ( \text{Since } N_{s^{\prime}}(\tau, a) \geq 1) 
\end{split}
\end{equation}

\end{proof}

As discussed in section 2-A, since the expected reward of an arm is in the range $[0,1]$, we get $\left( \mu^* - \mu(a) \right) \leq 1$ for all arms $a \in [1,K]$. Therefore, bounding the expected regret after $T$ pulls essentially amounts to controlling the expected number of times a sub-optimal arm is pulled. In the next lemma, we bound that expected number of total pulls $N_T(a)$ of some sub-optimal arm $a \neq a_t^*$ for the entire horizon $t \leq T$.

\begin{lemma}
Let $\xi > 1/2$. For any integer $\tau > 0$ and any arm $a \in [1,K]$, $$\mathbb{E}\left[ N_T(a) \right] \leq C(\tau)\frac{T\log(\tau)}{\tau} + \tau\beta_T + \log^2(\tau)$$where
$$C(\tau) = \frac{\left( l\sqrt{\tau} + 1\right)^2 \xi}{\left(\Delta_{\mu_T}(i)\right)^2}\frac{\lceil{T/\tau\rceil}}{T/\tau} + \frac{2}{\log(\tau)}\Bigg \lceil{ \frac{\log(\tau)}{\log(1+4\sqrt{1-(2\xi)^{-1}})} \Bigg \rceil}$$
\label{reg_swucb_thm}\end{lemma}

\begin{proof}
As in the analysis of Lemma 2, let us denote by $\Delta_{\mu_T}(a_t)$ the minimum of the difference of the expected reward $\mu_t(a^*_t)$ of the best arm at time $t$ and the expected reward $\mu_t(a_t)$ of the arm $a_t$ for all $t \in [1,T]$ when $a_t$ is not the optimal arm, i.e.,
\begin{equation}
\Delta_{\mu_T}(a_t) = \min_{t \in [1,T]; \ a_t \neq a^{*}_t} \left(\mu_t(a^{*}_t) - \mu_t(a_t)\right).    
\end{equation}
Let us consider some arm $a \in [1,K]$. The total pulls $N_T(a)$ can be represented as the following.

\begin{equation}
    N_T(a) = 1 + \sum_{t=K+1}^T \mathbf{1} \left(a_t=a \neq a_{t}^*\right)
\end{equation}

Which can be broken into two disjoint events, as mentioned below.

\begin{multline}
    N_T(a) = 1 + \sum_{t=K+1}^T \mathbf{1}\left(a_t=a \neq a_{t}^*;N_t(\tau,a) < A(\tau)\right) \\
+ \sum_{t=1}^T \mathbf{1}\left(a_t=a \neq a_{t}^*;N_t(\tau,a) \geq A(\tau)\right)
\end{multline}

where 
\begin{equation}
    A(\tau) =  \frac{\left(l\sqrt{\tau} + 1\right)^2 \xi \log(\tau)}{(\Delta_{\mu_{T}(a)})^2}
\end{equation}

The next few steps follow from \cite{gm}'s analysis. For the same definitions of $\mathcal{T}(\tau)$, and we can bound $N_T(a)$ by:
\begin{multline}
    N_T(a) \leq 1 + \lceil{T/\tau}\rceil A(\tau) + \tau \beta_T \\ + \sum_{t \in \mathcal{T}(\tau)} \mathbf{1} \left(a_t=a \neq a_{t}^*;N_t(\tau,a) \geq A(\tau)\right)
\end{multline}

Now, for $t \in \mathcal{T}(\tau)$ the event $E : \{a_t=a \neq a_{t}^*;N_t(\tau,a) \geq A(\tau)\}$ will occur when the following inequality, labelled as $\mathcal{Z}$ hold,

\begin{equation}
    \mathcal{Z} : \bar{X}_t(\tau, a) + c_t(\tau, a) > \bar{X}_{t}(\tau, a^*_t) + c_t(\tau, a^*_t)
\end{equation}

Let us denote by $\bar{Y}_t(\tau, a)$ the empirical average of arm $a$ \textit{excluding} the drift at each time until $t$ (i.e., pure empirical average, different from $\bar{X}_t(\tau, a)$) for the sliding window size of $\tau$. We have the following inequalities from $\mathcal{Z}$;

\begin{multline}
    \mathcal{Z} : \bar{Y}_t(\tau, a) + \frac{D_t(\tau, a)}{N_t(a)} + c_t(\tau, a) \\ > \bar{Y}_{t}(\tau, a^*_t) + \frac{D_t(\tau, a^*_t)}{N_t(a^*_t)} + c_t(\tau, a^*_t)
\end{multline}
 
Therefore, the upper bound for the difference between the expected reward of the optimal arm $a^{*}_t$ and the current arm $a$ is

\begin{equation} 
\begin{split}
\bar{Y}_{t}(\tau, a^*_t) - \bar{Y}_t(\tau, a) & < \frac{D_t(\tau, a)}{N_t(a)} + c_t(\tau, a) \\
 & = \left(l \sqrt{N_t(\tau, a)} + 1\right) \sqrt{\frac{\xi\log (\min(t,\tau))}{N_t(\tau,a)}}\\
 & \leq \left(l \sqrt{\tau} + 1\right) \sqrt{\frac{\xi\log (\min(t,\tau))}{N_t(\tau,a)}} 
\end{split}
\end{equation}

Now, we can decompose the event $E$ as the following: for the event $\mathcal{Z}$ to happen, at least one of the events  $E_t^i$ has to occur.

\begin{equation}
\{a_t=a \neq a_{t}^*;N_t(\tau,a) \geq A(\tau)\} \subseteq E_t^1 \cup E_t^2 \cup E_t^3
\end{equation}

Where 
\begin{equation}
    E_t^1 = \{ \bar{X}_t(\tau,a) > \mu_t(\tau,a) + c_t(\tau, a) \}
\end{equation}

\begin{equation}
    E_t^2 = \{ \bar{X}_t(\tau,a^*_t) < \mu_t(\tau,a^*t) - c_t(\tau, a^*_t) \}
\end{equation}
and
\begin{multline}
    E_t^3 = \{ \bar{Y}_{\tau,t}(a^*_t) - \bar{Y}_t(\tau,a) \\ <  \left(l \sqrt{\tau} + 1\right) \sqrt{\frac{\xi\log (\min(t,\tau))}{N_t(\tau,a)}} \}
\end{multline}

The event $E_t^1$ represents the situation when the algorithm overestimates the average reward of arm $a$, $E_t^2$ when the algorithm underestimates the average reward of the best arm $a^*_t$, and $E_t^3$ is when the expected rewards for both the arms $a$ and $a^*_t$ have a small difference.

Note that for the choice of $A(\tau)$, the event $E_t^3$ never occurs, as shown in the following inequality;
\begin{multline}
    \left(l \sqrt{\tau} + 1\right)\sqrt{\frac{\xi \log (\min(t,\tau))}{N_t(\tau,i)}} \\ \leq \left(l \sqrt{\tau} + 1\right)\sqrt{\frac{\xi \log (\tau))}{A(\tau)}} =  \frac{\Delta_{\mu_T}(i)}{2} 
\end{multline}

Therefore, we have the probability of the event $E_t^3$ as $\mathbb{P}\left[E_t^3\right] = 0$ and we get the corresponding probabilities for the $E_t^1$ and $E_t^2$ from \cite{gm} as 

\begin{equation}
    \mathbb{P}\left[ E_t^1 \right] = \mathbb{P}\left[ E_t^2 \right] \leq \Bigg \lceil{\frac{\log(\min(t,\tau))}{\log(1 + \eta)}} \Bigg \rceil \min(t,\tau)^{-2\xi \left(1 - \frac{\eta^2}{16}\right)}
\end{equation}

We finally have the bound, by taking $\xi > 1/2$ and $\eta = 4\sqrt{1 - (1/2\xi)}$, so as to make $2\xi(1 = \eta^2/16) = 1$:

\begin{equation}
    \begin{split}
        \mathbb{E}\left[ N_T(a) \right] & \leq \underbrace{1 + \lceil{T/\tau}\rceil A(\tau) + \tau \beta_T}_{M} + \underbrace{ 2\sum_{t=1}^T \frac{\bigg \lceil{\frac{\log(\min(t, \tau))}{\log(1+\eta)} } \bigg \rceil}{\min(t,\tau)}}_{N}\\
    \end{split}
\end{equation}

Substituting $A(\tau)$ in $M$ and expanding $N$ we upper bound  $\mathbb{E}\left[ N_T(a) \right]$ by

\begin{multline}
\mathbb{E}\left[ N_T(a) \right] \leq 1 + \lceil{T/\tau}\rceil \frac{\left(l\sqrt{\tau} + 1\right)^2 \xi \log(\tau)}{(\Delta_{\mu_{T}(i)})^2} + \\ \frac{2T}{\tau} \Bigg \lceil{ \frac{\log(\tau)}{\log(1+\eta)}  \Bigg \rceil} + \tau\beta_T + \log^2(\tau) 
\end{multline}
\end{proof}

With the proper choice of $\tau$, which minimizes the expression, we get the statement of Theorem 2.

\subsection{Material for Theorem 5}

Recall that in a continuously changing non-stationary environment, the mean rewards of the arms can change an arbitrary number of times but have a variation budget, limiting the total change throughout the horizon. 

In line 4 of Algorithm 4, in the main paper, the principal employs some stochastic bandit algorithm (such as  UCB1 or Thompson sampling) to recommend an arm to the agent. The standard instance-dependent regret results for some of the standard bandit algorithms have the form $O\left(\log T / \Delta_a\right)$ (\cite{auer2002finite}), where $\Delta_a$ is the difference in the expected reward of some arm $a$ and the best arm $a^*$ (also referred as the gap of arm $a$). For almost all intent and purposes, it is assumed that the minimum possible value of $\Delta_a$ (i.e., the denominator in the regret bound) is large enough to prevent the regret from being an arbitrarily large quantity and, therefore, make the regret bound meaningful. Assuming a well-behaved instance is reasonable for a stationary bandit model, however, in the continuously-changing non-stationary environment's case, it becomes a little too strong. Therefore, we need to give the rewards more freedom to vary but simultaneously impose certain assumptions to make them mathematically tractable. 

We start by breaking the time horizon $T$ into a sequence of batches of timestamps, where each batch (except possibly the last one) is of a fixed size $\sigma$. If we have total $m = \lceil T/\sigma \rceil$ batches, we define below the ``gap" of a certain arm $a$ in the current context. 

\begin{definition}
The minimum average difference $\Delta_j(a)$ between the mean rewards of some arm $a$ in comparison with the instantaneous best arm $a_t^*$ for each time step, within a single batch $j$, is defined as:
\begin{align*}
    \Delta_j(a) = \min_{j \in [1, m]} \frac{1}{\sigma} \sum_{t \in \mathcal{T}_j} (\mu_t^* - \mu_t(a))
\end{align*}
where $\mathcal{T}_j$ is the set of timestamps within the batch $j$. To be more precise, $\mathcal{T}_k = \{(k - 1)\sigma + 1,2,..., \min(T, k\sigma)\}$.
\end{definition}

In the next couple of assumptions, we form a balance between providing the rewards more freedom to vary and simultaneously imposing certain limitations to make them mathematically tractable. Assumption \ref{m-assump} lower-bounds the minimum average difference between the mean rewards of the best overall arm and any other arm within a single batch, and assumption \ref{alpha-assump} restricts the number of times the difference between the expected rewards between any pair of different arms can go below a certain threshold.

\begin{assumption}
\label{m-assump}
There exists a constant $M \in (0, 1)$ such that $\Delta_j(a) \geq M$ for any arm $a \in  [1,K]$ and all $j \in [1,m]$.
\end{assumption}

\begin{assumption}
\label{alpha-assump}
For any given time epoch $t$, there exists $\alpha \in (0, 1)$  and some threshold value $\varepsilon \in [0, 1]$ such that the following is true 
\begin{equation}
    \sum_{t=1}^T \sum_{a =1}^K \sum_{b=1; b \neq a}^K \mathbf{1} \left( \mu_t(a) - \mu_t(b) \leq \varepsilon \right) \leq t^{\alpha}
\end{equation}
\end{assumption}
 
Given the variation of the rewards (and potentially gaps of different arms), we would henceforth consider the regret bound, which does not depend on the said gap of an arm, but only the $T$ and $K$. For instance,  UCB1 has a worst-case regret bound (see chapter 1 in \cite{slivkins2021introduction}) of $O\left(\sqrt{KT \log T}\right)$ for $K$ arms. Next, we present the lemma that bounds the regret of algorithm 4 in the main paper, which directly leads to the main Theorem 5. 

\begin{lemma}
If the principal in Algorithm 4 (line 4) employs some stochastic bandit algorithm under reward drift with the worst-case regret of $\lambda \sqrt{TK \log (T)}$ for some constant $\lambda > 0$, and the batch size $\sigma = \left(\lambda T / V_T\right)^{2/3} (K \log (T))^{1/3}$, 
Then, for $T \geq 2, K \geq 2$ and $V_T \in [1/K, T/K]$, the total regret is
\begin{align}
    R_{\mathcal{V}}^T \leq 2 \lambda^{1/3} \cdot V_T^{1/3} \left(K \log(T) \right)^{1/3} T^{2/3}
\end{align}

\label{thm-cce}
\end{lemma}

\begin{proof} We use the following proof structure. We break the horizon into sequences of batches of size $\sigma$ each and then analyze the performance gap between the \textit{single} best action and the arm returned by a dynamic oracle (i.e., returns the best action for each time $t$) in each batch. Then, we plug in the known performance of the principal's bandit algorithm relative to the single best action under reward drift. We sum them over the batches to establish the required regret bound.
 
For the time horizon $T \geq 1$ and a total of $K \geq 2$ arms, we have the variation budget $V_T \in [1/K, T/K]$. We break the horizon $T$ into sequence of $m = \lceil{T/\sigma \rceil}$ batches denoted by $\mathcal{T}_j$ for all $j \in [1, m]$, of size $\sigma$, except possibly the last one. We decompose the regret in some batch $j$, as the expected sum of differences between the arm with the maximum expected reward $\mu_t^{*}$ at the time $t$ (i.e., dynamic oracle result) and the expected reward $\mu_t$ of an arm chosen by the algorithm at $t$, with $p = \mathbb{E}\left[ \max_{a \in [1,K]} \left\{ \sum_{t \in \mathcal{T}_j} X_t(a) \right\} \right]$ as:

 \begin{equation}
\mathbb{E}\left[ \sum_{t \in \mathcal{T}_j} \left( \mu_t^* - \mu_t \right) \right] = \underbrace{\sum_{t \in \mathcal{T}_j} (\mu_t^* - p)}_{J_1} +  \underbrace{p - \mathbb{E}\left[ \sum_{t \in \mathcal{T}_j} \mu_t \right]}_{J_2}
 \end{equation}

The first component $J_1$ represents the expected loss of the algorithm with respect to using a single action over batch $j$. The second component $J_2$ is the expected regret relative to the single best action in batch $j$. From (6) in \cite{besbes}, we know that $J_1 \leq 2V_j \sigma$. Since $J_2$ is, after all, the regret of a policy with respect to a single best action within a batch, we can plug in the worst-case performance of the bandit algorithm under reward drift. For some constant $\lambda > 0$, we have the following.
 
 \begin{equation}
     J_2 = \mathbb{E}\left[ \max_{a \in [1,K]} \left\{ \sum_{t \in \mathcal{T}_j} X_t(a) \right\} - \mathbb{E}\left[ \sum_{t \in \mathcal{T}_j} \mu_t^{\pi} \right] \right] \leq \lambda \sqrt{\sigma K \log (\sigma)}
 \end{equation}

The next step is to sum them over the entire horizon. Since there are $m$ batches, the overall regret is over all the permissible reward sequences $\mathcal{V}$:

\begin{equation}
    \begin{split}
        R_{\mathcal{V}}^T & = \sup_{\mu \in \mathcal{V}} \left \{\sum_{t=1}^T \mu_t^* - \mathbb{E} \left[ \sum_{t=1}^T \mu_t\right] \right\}\\
        & \leq \sum_{j=1}^{m} \left( \lambda \sqrt{\sigma K \log (\sigma)} +  2V_j\sigma \right)\\
        & \leq \left(\frac{2T}{\sigma} \right)\cdot \lambda \sqrt{\sigma K \log (\sigma)} + 2V_T \sigma \ (\text{since } m \geq 1)\\
        & = 2T \cdot \lambda \sqrt{\frac{K \log (\sigma)}{\sigma}} + 2V_T \sigma\\
    \end{split}
\end{equation}

Selecting { $\sigma = \left(\lambda T / V_T\right)^{2/3} (K \log (T))^{1/3}$}, we get 
{  
\begin{equation}
    R^{\pi}(\mathcal{V}, T)  \leq 2 \lambda^{1/3} \cdot V_T^{1/3} \left(K \log(T) \right)^{1/3} T^{2/3}
\end{equation}
}

\end{proof}

Therefore, to prove Theorem 5, by virtue of Lemma 6, we need to show that UCB1, $\varepsilon$-greedy and Thompson Sampling algorithms follow the required regret bound (for a batch) of $O\left(\sqrt{\sigma K \log \sigma}\right)$ under reward drift, for a batch size $\sigma$, and thus employable by the principal (i.e., in line 4 of Algorithm 4).

We know from \cite{liu20} that the regret for UCB1, $R^{\textit{UCB1}}$ under reward drift, for a batch $j$, for some constant $C$, is
\begin{equation}
    R^{\textit{UCB1}} \leq \sum_{a \in [1,K]; \Delta_j(a) > 0} \frac{C\left(l + 1\right)^2}{\Delta_j(i)} \log(\sigma)
\end{equation} 
As discussed before, the gap $\Delta_j(a)$ for any arm $a$ can get arbitrarily small in the denominator, therefore a more general regret bound should be derived (see remark 1.13 in \cite{slivkins2021introduction}), independent of the $\Delta_j(a)$ term. Let $\varepsilon \in (0,1)$, then
\begin{itemize}
    \item The regret contributed by all the arms with $\Delta_j(a) > \varepsilon$ is at most $\frac{CK\left(l+1\right)^2 \log(\sigma)}{\varepsilon}$.
    \item The regret contributed by all the arms with $\Delta_j(a) \leq \varepsilon$ is at most $\varepsilon\cdot\sigma^{\alpha}$ (from Assumption 2).
\end{itemize}
Therefore, the total regret for this batch $j$ is at most: 

\begin{equation}
    R^{\textit{UCB1}} \leq \varepsilon\cdot\sigma^{\alpha} + \frac{CK\left(l+1\right)^2 \log(\sigma)}{\varepsilon}
\end{equation}

For $\varepsilon = \sqrt{\frac{CK\left(l+1\right)^2 \log(\sigma)}{\sigma}}$ and $\alpha = 1$, we have $R_{\textit{UCB1}} \leq \sqrt{C}\left(l+1\right) \sqrt{\sigma K \log(\sigma)}$. This result is almost optimal as it almost matches the minimax lower bound of $O(\sqrt{\sigma})$ for any stochastic bandit algorithm (\cite{slivkins2021introduction}).

Next, from \cite{liu20} we know that the regret $R^{\varepsilon}G$ for $\epsilon$-greedy under reward drift, for a batch $j$, for some constant $C$, 
 
\begin{equation}
    R^{\varepsilon G} \leq \frac{l \cdot C}{M^2}\sum_{a \in [1,K]; \Delta_j(a) > 0} \frac{\log(\sigma)}{\left(\Delta_j(a)\right)^2}
\end{equation}
We employ the same technique as above to fix some $\varepsilon \in (0,1)$.
\begin{itemize}
    \item The regret contributed by all the arms with $\Delta_j(a) > \varepsilon$ is at most $\frac{KlC \log(\sigma)}{M^2 \varepsilon^2}$.
    \item The regret contributed by all the arms with $\Delta_j(a) \leq \varepsilon$ is at most $\varepsilon\cdot\sigma^{\alpha}$.
\end{itemize}
Therefore, the total regret for this batch $j$ is at most: 
{ 
\begin{equation}
    R^{\varepsilon G} \leq \varepsilon\cdot\sigma^{\alpha} + \frac{KlC \log(\sigma)}{M^2 \varepsilon^2}
\end{equation}
}
For $\varepsilon = \left(\frac{KlC \log(\sigma)}{M^2 \sigma}\right)^{1/3}$ and $\alpha = 3/4$, we have 
$R^{\varepsilon G} \leq \left(Cl/M^2\right)^{1/3} \left(K \log(\sigma)\right)^{1/3} \sqrt{\sigma} \leq \left(Cl/M^2\right)^{1/3} \sqrt{3\sigma K \log(\sigma)}$ for $\sigma \geq 1$. The analysis is the same as above for Thompson Sampling.

In summary, we showed that the principal could use UCB1, $\varepsilon$-greedy, and Thompson sampling as a submodule (line 4 in Algorithm 4) to recommend arms to the agent, balancing the exploration-exploitation dilemma and achieving a sublinear regret.

\subsection{Proof for Theorem 3}
Note that the principal compensates an agent when the following inequalities are true.

   The principal provides compensation when,
{ 
\begin{equation}
    \bar{X}_{t}(\gamma, a) > \bar{X}_{t}(\gamma, a^{*}_t)
\end{equation}
}
and 
{ 
\begin{equation}
    \bar{X}_{t}(\gamma, a) + \sqrt{\frac{\xi \log (n_t(\gamma))}{N_t(\gamma, a)}} < \bar{X}_{t}(\gamma, a^{*}_t) + \sqrt{\frac{\xi \log (n_t(\gamma))}{N_t(\gamma, a^{*}_t)}}
\end{equation}
}

Thus, the compensation is provided to the agent even when the agent pulls the optimal arm and { $N_t(\gamma, a^{*}_t) < N_t(\gamma, a)$}. Therefore, the average number of times a player is compensated for pulling the best arm is upper bounded by $M = \max_{a\neq a^{*}_t}\mathbb{E}[N_T(a)]$.
\begin{figure}
    \centering
    \includegraphics[width = 0.7\linewidth]{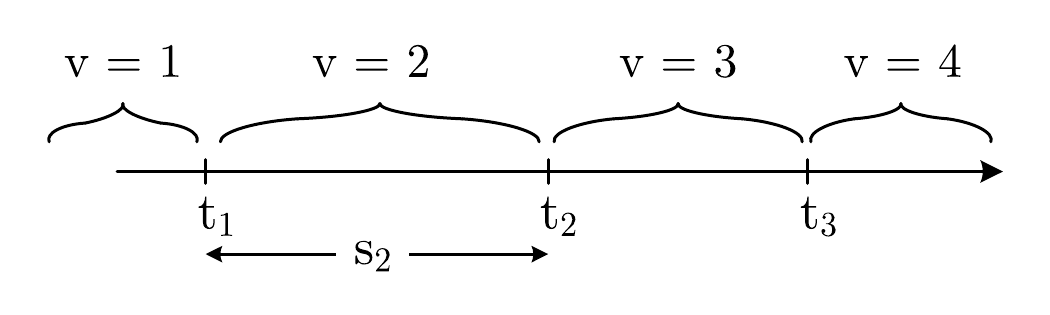}
    \caption{  Time intervals representing sets $V$ and $S$ for compensation analysis.}
    \label{fig:time-interval}
\end{figure}
Since the rewards can change with a breakpoint, so can the best arm. We must consider each interval between breakpoints separately.

Let $t_b$ be the timestamp at which the $b$th breakpoint occurs, $\forall b \in [1, \beta_T]$. Let $V = {\{(t_{b-1}, t_b) \ : \ \forall b \in [1, \beta_T]\}}$ and $S = {\{s_v \ : \ s_v = |t_{b-1} - t_b| \ \forall v \in V; \ \forall b \in [1,\beta_T]\}}$ (see Figure \ref{fig:time-interval}) be the set of intervals between breakpoints and the sizes of each interval, respectively. For convenience, consider $t_0 = 0$. 

Let $v^+ \in V$ be the interval in which the principal pays the compensation a maximum number of times  $M_{v^+}$ when the player plays the interval's best arm $i_v^*$. Let $\widehat{N}_{s_v}(a)$ be the number of times the suboptimal arm $a$ was played in the interval $v$. Therefore,

{  
\begin{equation}
    v^+ = \arg \max_{v \in V} \left[\max_{a \in [1,K]: a \neq a_v^*} E\left[ \widehat{N}_{s_v}(a)\right] \right]
\end{equation}
}

Taking $a^*$ to be the best arm in the interval $v^+$, We get, 
{ 
\begin{equation}
    \begin{split}
        M_{v^+} & \leq \beta_T \max_{a \neq a^*} \mathbb{E}\left[\widehat{N}_{s_v}(a)\right] \\
        & \leq \beta_T \max_{a \neq a^*} \mathbb{E}\left[\widehat{N}_T(a)\right]
    \end{split}
\end{equation}
}

We get the total expected compensation as
\begin{equation}
    \begin{split}
        \mathbb{E}[C_T] & \leq \sum_{i=1}^{K}\sum_{j=1}^{\mathbb{E}[\widehat{N}_T(a)]} \sqrt{\frac{\xi \log(n_t(\gamma))}{N_j(\gamma, i)}}\\
        & \leq \sum_{j=1}^{M_{v^{+}}}\sqrt{\frac{\xi \log(n_t(\gamma))}{N_j(\gamma, i)}}   +\sum_{i=1}^{K-1}\sum_{j=1}^{\mathbb{E}[\widehat{N}_T(i)]} \sqrt{\frac{\xi \log(n_t(\gamma))}{N_j(\gamma, i)}}\\
        & \leq \sqrt{\xi \log(n_t(\gamma))} \left( \sum_{j=1}^{\beta_T \max_{i \neq i^*} \mathbb{E}[\widehat{N}_T(i)]} 1 + \sum_{i=1}^{K-1} \sum_{j=1}^{\mathbb{E}[\widehat{N}_T(i)]} 1\right)\\
        & \leq \sqrt{\xi \log(n_t(\gamma))} \left( \sum_{i=1}^{K}\sum_{j=1}^{\beta_T \mathbb{E}[\widehat{N}_T(i)]} 1 + \sum_{i=1}^{K-1} \sum_{j=1}^{\mathbb{E}[\widehat{N}_T(i)]} 1\right)\\
        & \leq \sqrt{\xi \log(n_t(\gamma))} \left( \sum_{i=1}^{K}\sum_{j=1}^{(\beta_T+1) \mathbb{E}[\widehat{N}_T(i)]} 1\right)\\
        & \leq \sqrt{\xi \log(n_t(\gamma))} \left( \sum_{i=1}^{K} (\beta_T+1) \mathbb{E}[\widehat{N}_T(i)]\right)
    \end{split}
\end{equation}

\end{document}